\documentclass[3p,preprint,authoryear]{elsarticle}
\usepackage{tikz}
\usetikzlibrary{positioning,fit,arrows.meta,shapes,shapes.multipart,chains,scopes,matrix,intersections,calc}

\usepackage{listofitems} 
\usepackage[outline]{contour} 
\contourlength{1.4pt}

\usepackage{pgffor}
\usepackage{pgfplots}
\pgfplotsset{compat=1.18,colormap/hot}

\usetikzlibrary{external}
\newsavebox\mybox

\tikzset{
    font=\sf \scriptsize,
    >=LaTeX,
    cell/.style={
            rectangle,
            rounded corners=2.5mm,
            inner sep=5mm,
            draw,
            thick,
        },
    operator/.style={
            circle,
            draw,
            inner sep=0pt,
            minimum height =.5cm,
        },
    sum/.style= {operator,
            append after command={
                    (\tikzlastnode.north) edge (\tikzlastnode.south)
                    (\tikzlastnode.west) edge (\tikzlastnode.east)
                }
        },
    function/.style={
            ellipse,
            draw,
            inner sep=1pt
        },
    ct/.style={
            circle,
            draw,
            line width = .75pt,
            minimum width=0.75cm,
            inner sep=1pt,
        },
    gt/.style={
            rectangle,
            draw,
            minimum width=4mm,
            minimum height=3mm,
            inner sep=1pt
        },
    label/.style={
            font=\scriptsize\sffamily
        },
    ml/encoder/.style={trapezium, draw, trapezium angle=75, shape border rotate=270},
    ml/decoder/.style={trapezium, draw, trapezium angle=75, shape border rotate=90},
    ml/node/.style={draw, circle, minimum width=10pt},
    ml/weights/.style={thin, opacity=0.8},
    ml/weights/arrow/.style={-stealth, ml/weights},
    ml/weights/rounded1/.style={
            ml/weights,
            rounded corners=.25cm
        },
    ml/weights/rounded2/.style={
            ml/weights,
            rounded corners=.5cm
        },
    declare function={
            normalizing_constant(\sigmax,\sigmay,\rho) = 1 / (2 * pi * \sigmax * \sigmay * sqrt(1 - \rho^2));
            exponential_term(\x,\y,\mux,\muy,\sigmax,\sigmay,\rho) = -0.5 * (
            ((\x - \mux) / \sigmax)^2
            - 2 * \rho * ((\x -\mux) / \sigmax) * ((\y - \muy) / \sigmay)
            + ((\y - \muy) / \sigmay)^2
            ) / (1 - \rho^2);
            bivariate_gaussian(\x,\y,\mux,\muy,\sigmax,\sigmay,\rho) = normalizing_constant(\sigmax, \sigmay, \rho) * exp( exponential_term(\x, \y, \mux, \muy, \sigmax, \sigmay, \rho) );
        },
    every text node part/.style={align=center},
    node distance=1cm and 3mm,
    every on chain/.style=join,
    every join/.style={ml/weights/arrow},
}
\usepackage{amsmath,amsfonts,amsthm,bbold}
\usepackage{array}
\usepackage{textcomp}
\usepackage{url}
\usepackage{verbatim}
\usepackage{graphicx}
\usepackage{algorithm}
\usepackage{algpseudocode}
\usepackage{hyperref}
\usepackage{xspace}
\usepackage{xcolor}
\usepackage{physics}

\usepackage[subrefformat=parens,labelformat=parens]{subcaption}
\usepackage{caption} 

\usepackage{booktabs}
\usepackage{multirow}
\newcolumntype{R}[1]{>{\raggedleft\let\newline\\\arraybackslash\hspace{0pt}}m{#1}}
\newcolumntype{L}[1]{>{\arraybackslash}m{#1}}

\renewcommand{\vec}[1]{\mathbf{#1}}
\newcommand{\lhcgwconv}{LHC\_GW\_CONV\xspace}
\newtheorem{theorem}{Theorem}[section]
\newtheorem{lemma}[theorem]{Lemma}
\usepackage[binary-units]{siunitx}
\AtBeginDocument{\RenewCommandCopy\qty\SI}

\usepackage[colorinlistoftodos,color=red!40]{todonotes}
\makeatletter
\renewcommand{\todo}[2][]{\tikzexternaldisable\@todo[#1]{#2}\tikzexternalenable}
\renewcommand{\missingfigure}[2][]{MISSING FIGURE}
\makeatother

\usepackage{nomencl}
\makenomenclature

\usepackage{etoolbox}
\renewcommand\nomgroup[1]{%
\item[\bfseries
        \ifstrequal{#1}{A}{Acronyms}{%
        \ifstrequal{#1}{S}{Symbols}{}}%
]}

\DeclareSIUnit\km{\kilo\metre}
\DeclareSIUnit\kmsq{\km\squared}

\newcommand{\nomdef}[2]{%
    \ifcsname nom@\detokenize{#2}\endcsname
    #2
    \else
    \expandafter\gdef\csname nom@\detokenize{#2}\endcsname{}%
    #1 (#2)\nomenclature{#2}{#1}%
    \fi
}
\newcommand\tech[1]{\textit{#1}}
\renewcommand{\min}{\text{min}}
\renewcommand{\max}{\text{max}}
\newcommand{\ngauss}{{N_\textit{gaussian}}}
\newcommand{\rbuffer}{{R_\textit{buffer}}}
\newcommand{\fs}[2]{FS\_#1\_#2\xspace}
\newcommand{\fsppofcn}{\fs{PPO}{FCN}}
\newcommand{\fsppoconvtwod}{\fs{PPO}{CONV2D}}
\newcommand{\fsppolstm}{\fs{PPO}{LSTM}}
\newcommand{\fssacfcn}{\fs{SAC}{FCN}}
\newcommand{\fssacconvtwod}{\fs{SAC}{CONV2D}}
\newcommand{\fssaclstm}{\fs{SAC}{LSTM}}
\newcommand{\lstmaesac}{LSTMAE\_SAC\xspace}
\newcommand{\lstmaeppo}{LSTMAE\_PPO\xspace}
\newcommand{\lstmppo}{LSTM\_PPO\xspace}

\graphicspath{{sections/conclusion/figures/}{sections/intro/figures/}{sections/method/figures/}{sections/results/figures/}}

\begin{document}

\title{\LARGE \bf Recurrent Auto-Encoders for Enhanced Deep Reinforcement Learning in Wilderness Search and Rescue Planning}
\tnotetext[tfunding]{This work was supported by the Engineering and Physical
  Sciences Research Council, Grant/Award Number: EP/T517896/1-312561-05}
\author[1]{Jan-Hendrik Ewers \corref{cor1}}
\ead{j.ewers.1@research.gla.ac.uk}
\author[1]{David Anderson}
\ead{david.anderson@glasgow.ac.uk}
\author[1]{Douglas Thomson}
\ead{douglas.thomson@glasgow.ac.uk}
\affiliation[1]{organization={Autonomous Systems And Connectivity Research Division, University of
      Glasgow},
  country={Scotland}
}
\cortext[cor1]{Corresponding author}

\begin{abstract}
    Wilderness search and rescue operations are often carried out over vast landscapes. The search efforts, however, must be undertaken in minimum time to maximize the chance of survival of the victim. Whilst the advent of cheap multicopters in recent years has changed the way search operations are handled, it has not solved the challenges of the massive areas at hand. The problem therefore is not one of complete coverage, but one of maximizing the information gathered in the limited time available.
    In this work we propose that a combination of a recurrent autoencoder and deep reinforcement learning is a more efficient solution to the search problem than previous pure deep reinforcement learning or optimisation approaches. The autoencoder training paradigm efficiently maximizes the information throughput of the encoder into its latent space representation which deep reinforcement learning is primed to leverage. Without the overhead of independently solving the problem that the  recurrent autoencoder is designed for, it is more efficient in learning the control task. We further implement three additional architectures for a comprehensive comparison of the main proposed architecture. Similarly, we apply both soft actor-critic and proximal policy optimisation to provide an insight into the performance of both in a highly non-linear and complex application with a large observation space.
    Results show that the proposed architecture is vastly superior  to the benchmarks, with soft actor-critic achieving the best performance. This model further outperformed work from the literature whilst having below a fifth of the total learnable parameters and training in a quarter of the time.
\end{abstract}

\maketitle

\section{Introduction}
\label{sect:intro}

\nomdef{Wilderness search and rescue}{WiSAR} missions are some of the most time-sensitive operations in existence. Shaving off seconds in the time to find and the resultant rescue can directly result in saved lives. Over small areas it can be effective to quickly cover the entire search space using modern technology such as drones \citep{carrell_flying_2022}, however this becomes intractable over larger areas. The search area can quickly balloon into the tens of kilometres in width and depth when considering a WiSAR scenario. This introduces the requirement to take the endurance of the searcher into account as complete coverage is no longer feasible. This is referred to as search planning which aims to maximize a objective given a maximum path length.

Current approaches to using \nomdef{unmanned aerial systems}{UAS} during deployment by organisations like Police Scotland Air Support Unit is the pilot-observer model. This mandates that there are always at least two personnel present to operate the UAS no matter the scenario. The pilot flies the drone whilst also observing and inspecting the live camera feed, whereas the observers is in charge of maintaining a visual line-of-sight to the UAS at all times for safety and legislative reasons. In work by \citet{koester_sweep_2004}, it was identified that a searcher has a higher probability of detection when not in motion leading to the stop-and-look method. Whilst this work was carried out for foot-based searchers, the same strategy can be observed in \citet{ewers_optimal_2023} for UAS pilots. It is therefore evident that the cognitive load of manoeuvring and searching is a key limitation, and that being able to offload the menial flying of the drone frees up the pilot to spend more effort on the search.

Previous work \citep{ewers_optimal_2023, ewers_deep_2025}, has shown the strength of \nomdef{Deep Reinforcement Learning}{DRL} over analytical and optimal search planning methods.
Work by \citet{talha_autonomous_2022} and \citet{peake_deep_2021} applied DRL to a similar problem, however these algorithms explore the
environment during the search and do not have the complete \nomdef{Probability Distribution Map}{PDM} at the beginning. This is a reasonable scenario to be in at the start of the search as local area knowledge, maps, and case studies are all available to a pilot before the search begins. Similarly, PDM generation algorithms such as from \citet{hashimoto_agentbased_2022} are viable solutions to generate the PDM as an input.

The proposed algorithm builds on the previously efforts seen in \citep{ewers_deep_2025} which had severe limitations due to the policy network architecture.
The most major of these limitations was that through the use frame-stacking the maximum number of waypoints in a search path was constrained.
Another limitation was the long training time required to train the policy due to its large amount of learnable parameters.
This of course also results in poor runtime performance during deployment as multiple gigabytes of memory are consumed to hold the model on potentially resource limited devices.

\citet{mock_comparison_2023} found that using frame-stacking over recurrent architectures leads to comparable performance.
However, in this formulation the observations from early in the simulation become decreasingly unimportant. In the search planning algorithm from \citet{ewers_deep_2025} this is not the case and every step has an equal impact on the next steps reward. Hence, the policy must be able to observe the observations back to $t=0\si\second$ to maintain the \nomdef{Markv Decision Process}{MDP}.

In this work we aim to tackle two of the aforementioned problems: find an alternative to frame-stacking, and increase general performance (training times, model size, overall efficacy).
In \citet{raffin_decoupling_2019} a \nomdef{Auto-Encoder}{AE} is successfully used within the DRL loop whilst in \citep{park_sequencetosequence_2018} the sequence-to-sequence architecture is introduced for text-based applications. This is the basis to approach the first problem, and also a possible approach to handle the second. We hypothesis that by splitting out the feature extraction phase as a observation preprocessing step we can harness the powerful AE training setup, and thusly reduce the training overhead by not having to learn the feature extraction during the DRL phase. This will then result in two models (AE and DRL policy) working in unison that are specialized and substantially smaller than in \citet{ewers_deep_2025} with better performance.

This work therefore contributes the following advances to the field:
\begin{itemize}
    \item A framework to decrease training times and to significantly reduce the number of learnable parameter, whilst enhancing final performance through the proposed AE and DRL architecture,
    \item Empirical evaluation of frame-stacking and recurrent frameworks for large observations,
    \item Further the discussion in comparing PPO and SAC in a large observation space domain.
\end{itemize}

Related work is discussed in \autoref{sect:related_work}. The environmental modelling is presented in \autoref{sect:method}, whilst the RAE and DRL architectures are introduced in \autoref{sect:rae} and \autoref{sect:architecture} respectively. Implementation details are outlined in \autoref{sect:implementation}, followed by results and discussions in \autoref{sect:results} and \autoref{sect:discussion}. Finally, a conclusion is drawn in \autoref{sect:conclusion}.



\section{Related Work}
\label{sect:related_work}

\nomdef{Deep Reinforcement Learning}{DRL} has seen significant application across exploration planning domains.
\citet{zuluaga_deep_2018} focuses on urban scenarios and incorporates frontier exploration into the search task where the agent gathers information on the environment over time. \citet{niroui_deep_2019}  uses SLAM with DRL to explore a cluttered environment, which has many parallels to mining search, in real time. Similarly \citet{peake_deep_2021}  applies DRL to WiSAR and employs dual-policy DRL, DDPG and recurrent A2C, which were trained separately to handle the exploration and trajectory planning. \citet{talha_autonomous_2022} also uses two DRL policies to handle the navigation and exploration separately. These foundational works notably omit consideration of prior probability distributions for target locations. \citet{ewers_deep_2025} explored this scenario and showed that DRL outperforms search planning methods from the literature, however with long training times, and too many parameters to be practical.

One potential solution to reducing the training overhead is to use hierarchical DRL with multiple more specialized models and policies. As previously outlined, \citet{peake_deep_2021} and \citet{talha_autonomous_2022} apply the dual-policy paradigm but still rely on a single policy to handle the top-level exploration planning. This then implies that the search planning from \citet{ewers_deep_2025} can be coupled with the lower-level trajectory or navigation planning from the aforementioned work. However, the same issues arise in that the mission planning policy is no better than before.

The architecture in \citet{ewers_deep_2025} uses frame-stacking to ensure that the \nomdef{Markov Decision Process}{MDP} is maintained. This is due to the reward at the current time step being dependent on the position of the agent in all previous time steps. If the agent intersects with the historical path then it is penalized by not gaining any new information. Whilst \citet{mock_comparison_2023} found that frame-stacking and recurrent architectures performed similarly, frame-stacking imposes a hard upper limit on the size of the input. The buffer size can be increased to overcome this issue but this has problems of its own. If the buffer is not full during training then inputs associated with data points far in the future cannot be trained leading to inefficient - or even unstable - training. It is therefore imperative to find another approach to handle the temporal input.

In \citet{raffin_decoupling_2019}, the concept of decoupling the feature extraction for DRL is explored. It was found that the proposed method was far superior to the standard single-policy approach in DRL. Interestingly, it was found that this method was only slightly better than an \nomdef{Auto-Encoder}{AE}. AEs have two components: an encoder applying a transformation on the input into a latent space, followed by a decoder approximating the reverse of this process \citet{berahmand_autoencoders_2024}. However, the method from this work requires frame-stacking again to work effectively.
\citet{pleines_generalization_2022} used a \nomdef{Long-Short Term Memory}{LSTM} layer \citep{hochreiter_long_1997} within the policy but another approach is to couple this with the aforementioned AE to leverage the sequence-to-sequence architecture proposed in work by \citet{park_sequencetosequence_2018} to predict the trajectory of a vehicle through a \nomdef{Recurrent Auto-Encoder}{RAE}. The sequence-to-sequence architecture was also used in \citet{cho_learning_2014} to process complex phrase representations for translations; another domain where there are dependencies on the entire variable length dataset for context.

\nomdef{Proximal Policy Optimisation}{PPO} \citep{schulman_proximal_2017}, its recurrent variant \nomdef{Recurrent PPO}{RPPO} \citep{raffin_stablebaselines3_2021,pleines_generalization_2022}, and \nomdef{Soft Actor-Critic}{SAC} \citep{haarnoja_soft_2019} are widely used in control problems such as by \citet{kaufmann_championlevel_2023} and \citet{yue_deep_2022}, as well as in other domains such as video gaming \citep{openai_dota_2019}.
PPO's stability and simplicity make it ideal for initial policy convergence in deterministic settings, while SAC's entropy-driven exploration excels in dynamic environments requiring adaptive action distributions \citep{shianifar_optimizing_2025}. The contrast between PPO's bounded policy updates (via advantage function clipping) and SAC's stochasticity (via the maximum entropy formulation) provides a methodological spectrum to evaluate robustness. PPO has become one of the de facto DRL algorithms in the literature, however \citet{mock_comparison_2023} found that it was unable to cope with higher dimensional observation spaces as well as SAC could.

Our work addresses the training instability of frame-stacking through a RAE architecture that compresses temporal dependencies into latent states. By integrating sequence-to-vector trajectory encoding with decoupled feature extraction, we enable dynamic adaptation to environmental uncertainty while maintaining compatibility with the dimensionality constraints of the policy networks. This approach uniquely resolves the conflict between long-horizon probabilistic reasoning and fixed-size observation spaces.

\newcommand{\spdm}{
    \left[
        \vec\mu_0,
        \vec\sigma_0,
        \dots,
        \vec\mu_G,
        \vec\sigma_G
        \right]^T
}
\newcommand{\spath}{
    \left(
    \vec x ~\|~\mathrm{0}^{2\times N_\mathrm{waypoints}-t}
    \right)^T
}

\section{Environment Modelling}
\label{sect:method}

\subsection{Agent Dynamics}

The agent within the environment is modelled as a heading control model with a fixed step size $s \unit{\meter}$.
It is assumed that any physical vehicle, such as a drone in the case of WiSAR, executing this mission can accurately track the waypoints via its controller or operator.
The agent's state is represented by the 2D position vector $\vec x = [x,y]^T \in \mathbb{R}^2$ and it's dynamics are described by the nonlinear system
\begin{align}
    \dot{\vec x}
     & =
    s
    \begin{bmatrix}
        \cos (\pi (a_t+1)) \\
        \sin (\pi (a_t+1))
    \end{bmatrix} \\
    \vec y
     & =
    \begin{bmatrix}
        1 & 0 \\ 0& 1
    \end{bmatrix}
    \begin{bmatrix}
        x(t) \\ y(t)
    \end{bmatrix}
\end{align}
where $a_t \in [-1,1]$ is the agent's action at time-step $t$ and the Euler integration scheme, $\vec x_{t+1} = \vec x_t + \delta t \dot{\vec x}$,  is used with $\delta t = 1 \unit{\second}$.

\subsection{Probability Distribution Map}


It is assumed that the \nomdef{Probability Distribution Map}{PDM} is known and is modelled as a sum of $\ngauss$ bivariate Gaussians \citep{ewers_enhancing_2024}.
\citet{yao_gaussian_2017} and \citet{yao_gaussian_2022} showed that bivariate Gaussians are effective at approximating PDMs.
The agent at position $\vec x$ thus has a probability of being above the true search objective, a missing person for WiSAR, given by
\begin{equation}
    p(\vec x) =
    \frac{1}{\ngauss}
    \sum^\ngauss_{i=0}
    \frac
    {
        \exp{
            \left[
                -\frac
                {1}
                {2}
                (\vec x - \vec \mu_i )^T\vec\sigma_i^{-1}(\vec x - \vec \mu_i)
                \right]
        }
    }
    {\sqrt{4\pi^2\det{\vec\sigma_i}}}\label{eqn:sum_of_bivariate_gaussians}
    \\
\end{equation}
where $\vec \mu_i \in \mathbb{R}^2$
and
$\vec \sigma_i \in \mathbb{R}^{2 \times 2}$
are the mean location and covariance matrix of the $i$th bivariate Gaussian respectively. The mean is reset at the start of every episode with
\begin{equation}
    \vec \mu_i \sim \mathcal{U}([x_\min , x_\max],
    [y_\min, y_\max]
    ),\forall i \in [0,\ngauss]\label{eqn:method:sampling_mus}
\end{equation}

The covariance is left unchanged to avoid a further tunable simulation parameter. \autoref{fig:method:pdm} shows how the various PDMs are still highly irregular even with constant covariance due to the randomness introduced by \autoref{eqn:method:sampling_mus}.

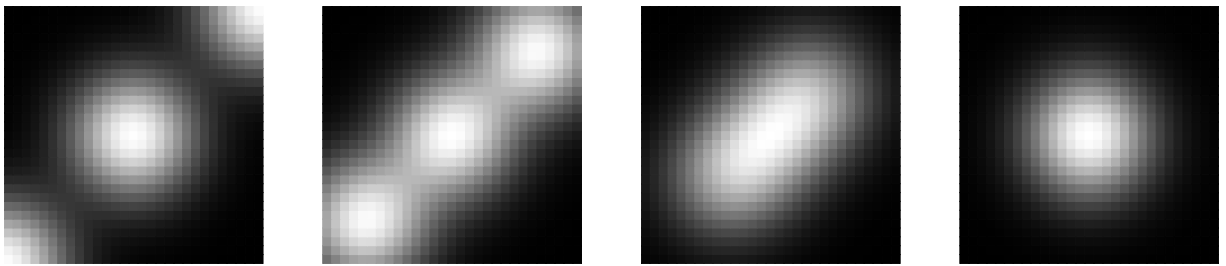
\begin{figure}[htbp]
    \centering
    \newcommand{\threetwodgauss}[4]{
        \begin{tikzpicture}[
            ]
            \begin{axis}[
                    axis equal image,
                    scale only axis,
                    xlabel=$x$,
                    ylabel=$y$,
                    zlabel=$z$,
                    hide axis,
                    view={0}{90},
                    width=#4,
                    colormap={bw}{
                            gray(0cm)=(0);
                            gray(1cm)=(1);
                        },
                ]
                \addplot3[
                    surf,
                    shader=flat,
                    samples=30,
                    domain=0:15,
                ]
                {
                    (1/3)*(
                    bivariate_gaussian(x, y, #1, #1, 2.5, 2.5, 0) +
                    bivariate_gaussian(x, y, #2, #2, 2.5, 2.5, 0) +
                    bivariate_gaussian(x, y, #3, #3, 2.5, 2.5, 0)
                    )
                };
            \end{axis}
        \end{tikzpicture}
    }

    \subfloat[The three modes are far apart with the saddle being close to $0$.]{
        \threetwodgauss{0}{7.5}{15}{0.24\linewidth}
    }
    \hfill
    \subfloat[The saddle is almost the same value as the maxima of the modes.]{
        \threetwodgauss{2.5}{7.5}{12.5}{0.24\linewidth}
    }
    \hfill
    \subfloat[All three maxima have merged into a pill-shaped area of high value.]{
        \threetwodgauss{5}{7.5}{10}{0.24\linewidth}
    }
    \hfill
    \subfloat[All maxima are aligned leading to a single hotspot.]{
        \threetwodgauss{7.5}{7.5}{7.5}{0.24\linewidth}
    }
    \caption{Example PDM $p(\vec x)$ with $\ngauss = 3$ and constant covariance. Lighter areas are of higher probability whilst darker ones have lower probability. During search planning the agent would avoid lower probability regions whilst targeting high probability ones to maximize total seen probability.}
    \label{fig:method:pdm}
\end{figure}

\subsection{Reward Architecture}
\label{sect:method:rewards}

As the agent moves a constant distance $s\si\meter$ every step, it is assumed that the camera follows this path continuously at a fixed height whilst pointing straight down at all times.
Therefore, to represent the \tech{seen area} for a given path at time-step $t$, the path is buffered by $\rbuffer\si\meter$ to give the polygon $H_t$.
All probability from the PDM enclosed within $H_t$ is then \tech{seen} and denoted by $p_t$.
This value, the seen probability, is calculated through
\begin{equation}
    I(C) = \oint_C f(\vec x) dC
\end{equation}
Substituting $C=H_t$ and \autoref{eqn:sum_of_bivariate_gaussians} gives
\begin{equation}
    p_t = \oint_{H_t} p(\vec x) d H_t
    \label{eqn:accumulated_probability_at_t}
\end{equation}

Our goal is to maximize the captured probability mass.  To gain insight into how the agent's path affects this, we first focus on the local behaviour of $p_t$. We analyze the area covered by the agent in two steps, as this provides a foundation for understanding more complex paths.  The following lemma demonstrates a crucial property of this two-step area in the simplified case of a uniform PDM.
\begin{lemma}
    \label{lemma:method:pdm_optimal_theta_0}
    For a uniform PDM, $p(\vec x) = 1$, the area $A(\theta)$ of the region $H$ after two steps, as defined by
    \begin{equation}
        \begin{gathered}
            A_2(\theta) =
            \overbrace{4 s \rbuffer}^\textrm{Main Area}
            +\overbrace{\pi \rbuffer^2}^\textrm{Semi-Circle}
            +\overbrace{\frac{1}{2} \rbuffer^2 \theta}^{\textrm{Rounded corner}~g(\theta)}
            -\overbrace{
                \rbuffer^2 \tan
                \left(
                \frac{\theta}{2}
                \right)
            }^{\mathrm{Overlap}~f(\theta)} \\
            \forall~\theta \in \left[0, 2 \arctan\left(\frac{s}{\rbuffer}\right)\right]
            \label{eqn:method:area_of_two_steps}
        \end{gathered}
    \end{equation}
    where $\rbuffer$ and $s$ are constants, is maximized when $\theta = 0$.
\end{lemma}
\begin{proof}
    Assume, for the sake of contradiction, that there exists a $\theta^* \in \left[0, 2 \arctan\left(\frac{s}{\rbuffer}\right)\right]$ such that $A_2(\theta^*) > A_2(0)$.
    The derivative of $A_2$ with respect to $\theta$ is
    \begin{equation}
        \dv{A_2}{\theta}
        = \frac{1}{2} \rbuffer^2 \left(1-\sec^2\left(\frac{\theta}{2}\right)\right)
    \end{equation}
    Since $1- \sec^2(x) \leq 0~\forall~ x \in \mathbb{R}$, we have $\dv{A_2}{\theta} \leq 0~\forall~\theta \in \left[0, 2 \arctan\left(\frac{s}{\rbuffer}\right)\right]$.
    Because the derivative is non-positive, $A_2(\theta)$ is a monotonically decreasing function on the given interval.
    Since $A_2(\theta)$ is monotonically decreasing, for any $\theta^* > 0$, it must be the case that  $A_2(\theta^*) \leq A_2(0)$.
    Our assumption that there exists a $\theta^*$ such that $A_2(\theta^*)>A_2(0)$ must be false. Therefore, the maximum value of $A_2(\theta)$ is achieved when $\theta=0$.
\end{proof}
Further insights can be garnered by applying Green's theorem
\begin{equation}
    \oint_C (Ldx+Mdy)=\int \int_D (\pdv{M}{x} - \pdv{L}{y}) dA
\end{equation}
with $\pdv{M}{x} - \pdv{L}{y}=1$. This shows that decreasing the boundary $C=H_2$ reduces the area of the region $D$ bounded by $C$. With a uniform PDM, maximizing the geometric area is equivalent to maximizing the captured probability mass. From \autoref{lemma:method:pdm_optimal_theta_0}, the buffered polygon formulation maximizes the integral when $\theta=0$ for a uniform PDF. However, if $\pdv{M}{x} - \pdv{L}{y}$ is not constant it could be beneficial to increase $\theta$ and therefore reducing the area in order to maximize the encapsulated values.

Special consideration must be taken for the case where $\frac{s}{\rbuffer} < \frac{\pi}{2}$ as \autoref{lemma:method:pdm_optimal_theta_0} does not hold and must be further explored. This constraint, however, is always met in this work.

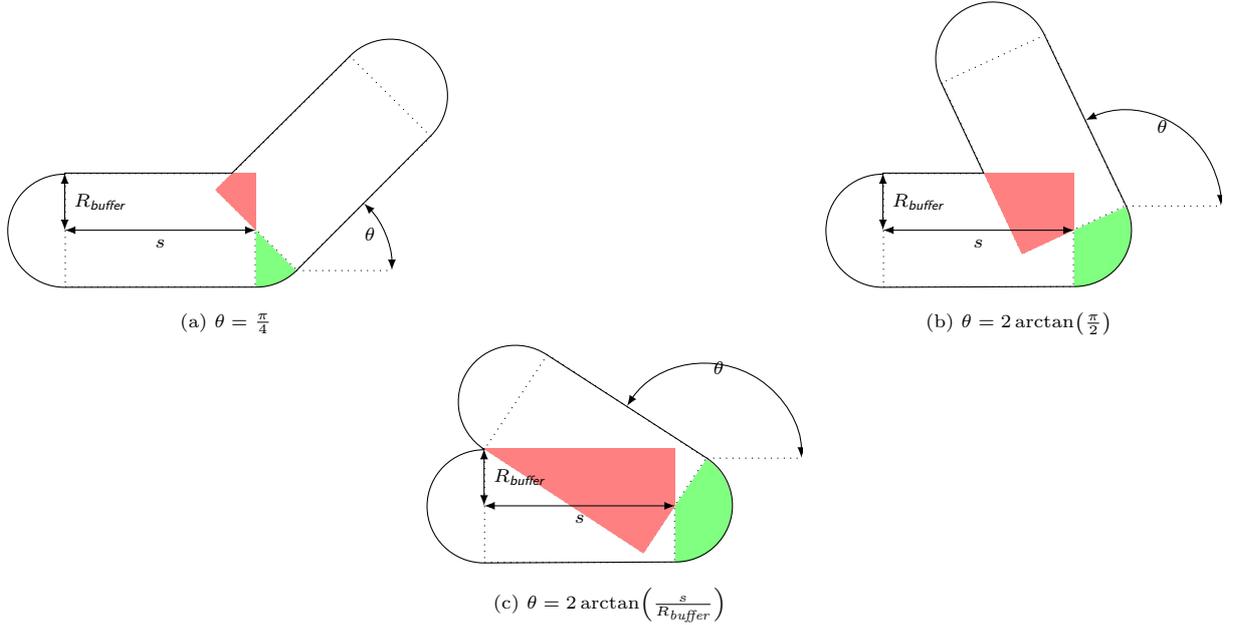
\begin{figure}[htbp]
    \newcommand{\bufferedlinestring}[3]{
        \begin{tikzpicture}[]
            \def\angle{#1};
            \def\radius{#2};
            \def\length{#3};

            \node[rectangle, draw, dotted, minimum height=2*\radius cm, minimum width=\length cm, anchor=west] (A)  {};
            \node[rectangle, draw, dotted, minimum height=2*\radius cm, minimum width=\length cm, rotate=\angle, anchor=west] (B) at (A.east) {};

            \node[rectangle, fit=(A) (B)] (surround) {};
            \begin{scope}
                \path[clip] (A.north west) -- (A.north east) -- (A.south east) -- (A.south west) --cycle;
                \path[clip] (B.north west) -- (B.north east) -- (B.south east) -- (B.south west) --cycle;
                \fill[red!50] (surround.south west) rectangle (surround.north east);
            \end{scope}
            \fill[green!50,] (A.east) -- (A.south east) arc[start angle=270, delta angle=\angle, radius=\radius] --cycle;
            \draw[dotted] (B.south west) -- +(\length/2,0);
            \draw[latex-latex] (B.south) arc[start angle=\angle, delta angle=-\angle, radius=\length/2] node[midway,left] {$\theta$};
            \coordinate (ABNorthIntersect) at (intersection of  A.north west--A.north east and B.north west--B.north east);

            \path[draw] (A.north west)
            -- (ABNorthIntersect)
            -- (B.north east) arc[start angle=90+\angle, delta angle=-180, radius=\radius]
            -- (B.south west) arc[start angle=270+\angle, delta angle=-\angle, radius=\radius]
            -- (A.south west) arc[start angle=270, delta angle=-180, radius=\radius]
            --cycle;
            \draw[latex-latex] (A.north west) -- node[right,right] {$\rbuffer$} (A.west);
            \draw[latex-latex] (A.west) -- node[midway,below] {$s$} (A.east);
        \end{tikzpicture}
    }
    \centering
    \subfloat[
        $\theta=\frac{\pi}{4}$
    ]{
        \bufferedlinestring{45}{0.75}{2.5}
    }
    \hfill
    \subfloat[
        $\theta = 2\arctan(\frac{\pi}{2})$
    ]{
        \bufferedlinestring{2*atan(pi/2)}{0.75}{2.5}
    }
    \hfill
    \subfloat[
        $\theta = 2\arctan(\frac{s}{\rbuffer})$
    ]{
        \bufferedlinestring{2*atan(2.5/0.75)}{0.75}{2.5}
    }
    \caption{Visualizing $H_2$ ($H_t$ after two steps) with different $\theta$. The areas coloured in red and green represent $f(\theta)$ and $g(\theta)$ from \autoref{eqn:method:area_of_two_steps} respectively.}
    \label{fig:method:buffered_linestring}
\end{figure}

The action is correlated to the reward by considering the change in accumulated probability at time $t$, defined as
\begin{equation}
    \Delta p_t = p_t - p_{t-1}\label{eqn:delta_p_t}
\end{equation}
To normalize this change in accumulated probability we introduce the scaling constants $k$ and $p_A$.
Constant $k$ relates the area of a single isolated step to the total search area, $A_\mathrm{area}\unit{\meter^2}$.
Simplifying \autoref{eqn:method:area_of_two_steps} to the single step case using the constants defined in \autoref{fig:method:buffered_linestring}, gives the ratio
\begin{equation}
    k = \frac{A_\textit{area}}{\rbuffer(\pi\rbuffer+2s)}
    \label{eqn:k}
\end{equation}
Constant $p_A$ is the total probability enclosed within the total search area given by substituting $H_t = A$ in \autoref{eqn:accumulated_probability_at_t}.
This is gives the scaled probability efficiency reward
\begin{equation}
    r = \frac{k}{p_A} \Delta p_t
    \label{qen:method:scaled_probability_effectiency_reward}
\end{equation}
The ratio of change in accumulated probability to total probability enclosed within the search area satisfies the constraint that $0 \leq \frac{\Delta p_t}{p_A} \leq 1$. This ratio is the probability efficiency, $e_{p,t}$, and provides a useful insight into the performance of a given path.

Further reward shaping is applied to discourage future out-of-bounds actions ($w_{oob}$), and to penalize visiting areas of very low probability ($w_0$).
The augmented reward $r'$ is given by
\begin{equation}
    r' =
    \begin{cases}
        -w_{oob},   & \vec x_t \notin [x_\min,x_\max] \times [y_\min,y_\max] \\
        w_r r ,     & \Delta p_t > \epsilon                                  \\
        -w_{0}    , & \textit{else}
    \end{cases}
    \label{eqn:reward_with_cases}
\end{equation}

\subsection{Observation Processing}

The observation vector at time $t\unit{\second}$ is denoted by $s_t$. To ensure flexibility when designing the architectures, the available sub-states are given in \autoref{tbl:observation_states} with architecture-specific observation space definitions given in \autoref{sect:architecture}.

\begin{table}[htb]
    \centering
    \caption{Definition of the five state observations}
    \label{tbl:observation_states}
    \begin{tabular}{@{}lll@{}}
        \toprule
        Sub-state       & Symbol             & Definition                                            \\ \midrule
        Path            & $s_\mathrm{path}$  & $\spath$                                              \\
        PDM             & $s_\mathrm{PDM}$   & $\spdm$                                               \\
        Position        & $s_\mathrm{pos}$   & $\vec x_t$                                            \\
        Out-of-bounds   & $s_\mathrm{oob}$   & $\vec x_t \in [x_\min,x_\max] \times [y_\min,y_\max]$ \\
        Number of steps & $s_\mathrm{steps}$ & $t$                                                   \\ \bottomrule
    \end{tabular}
\end{table}

\section{Recurrent Autoencoder}
\label{sect:rae}

Recurrent encoders project a multidimensional input sequence to a fixed-length latent space $z_t$ through $E_\phi(x_t) \mapsto z_t$, parametrized by $\phi$ \citep{cho_learning_2014}. This work this uses the \nomdef{Long-Short Term Memory}{LSTM} architecture \citep{hochreiter_long_1997}. LSTM networks have a hidden state $h_t$ that is passed through the layers which holds the memory of previously seen states whereas $c_t$ is similar in that it carries the information about the sequence over time but is unique to each cell. Both $h_t$ and $c_t$ are critical to the handling of sequential data.
For each element in the input sequence, each layer computes the function
\begin{align}
    \label{eqn:lstm}
    \begin{split}
        i_t & = \sigma(W_{xi}x_t + W_{hi}h_{t-1} + b_{i}) \\
        f_t & = \sigma(W_{xf}x_t + W_{hf}h_{t-1} + b_{f}) \\
        g_t & = \tanh(W_{xg}x_t +  W_{hg}h_{t-1} + b_{g}) \\
        o_t & = \sigma(W_{xo}x_t + W_{ho}h_{t-1} + b_{o}) \\
        c_t & = f_t \odot c_{t-1} + i_t \odot g_t         \\
        h_t & = o_t \odot \tanh(c_t)
    \end{split}
\end{align}
where $i_t$, $f_t$, $g_t$, and $o_t$ are the input, forget, cell, and output gates respectively, $\sigma$ is the sigmoid function, $\odot$ is the hadamard product, and  $W$ and $b$ are the parameter matrices and vectors. $h_{t-1}$ is the hidden state of the previous layer at time $t-1$ and is initialized at time $t=0$ to be zero.
The \nomdef{Long Short-Term Memory}{LSTM} unit internal structure can be seen in \autoref{fig:rae:lstm_unit} showing the three gates interacting with the various states. $y_t$ is the output and is equal to $h_t$ of the final layer if multiple layers are used.

\begin{figure}[htbp]
    \centering
    \begin{tikzpicture}[]
    \node [function] (sig1) at (-2,-0.75) {$\sigma$};
    \node [function] (sig2) at (-1.5,-0.75) {$\sigma$};
    \node [function] (sig3) at (0.5,-0.75) {$\sigma$};
    \node [function] (tanh1) at (-0.5,-0.75) {$\tanh$};
    \node [function] (tanh2) at (1.5,0.75) {$\tanh$};

    \node [operator] (add1) at (-0.5,1.5) {+};
    \node [operator] (hadmard1) at (-2,1.5) {$\cdot$};
    \node [operator] (hadmard2) at (-0.5,0) {$\cdot$};
    \node [operator] (hadmard3) at (1.5,0) {$\cdot$};

    \node[ct, label={[label]Cell}] (ctprev) at (-4,1.5) {$c_{t-1}$};
    \node[ct, label={[label]Hidden}] (htprev) at (-4,-1.5) {$h_{t-1}$};
    \node[ct, label={[label]left:Input}] (xt) at (-2.5,-3) {$x_t$};
    \node[ct, label={[label]Cell}] (ct) at (4,1.5) {$c_t$};
    \node[ct, label={[label]Hidden}] (ht) at (4,-1.5) {$h_t$};
    \node[ct, label={[label]left:Output}] (yt) at (2.5,3) {$y_t$};

    \draw[ml/weights/rounded1] (ctprev) -- (hadmard1) -- (add1) -- node[pos=0.2, below] {$o_t$} (ct);

    \draw [ml/weights/rounded1] (xt) -- (xt |- htprev)-| (sig1);
    \draw [ml/weights/rounded1] (xt) -- (xt |- htprev)-| (sig2);
    \draw [ml/weights/rounded1] (xt) -- (xt |- htprev)-| (sig3);
    \draw [ml/weights/rounded1] (xt) -- (xt |- htprev)-| (tanh1);
    \draw [ml/weights/rounded1] (htprev) -| (sig1);
    \draw [ml/weights/rounded1] (htprev) -| (sig2);
    \draw [ml/weights/rounded1] (htprev) -| (sig3);
    \draw [ml/weights/rounded1] (htprev) -| (tanh1);

    \draw [->, ml/weights/rounded1] (tanh1) -- node[right] {$g_t$} (hadmard2);
    \draw [->, ml/weights/rounded1] (hadmard2) -- (add1);
    \draw [->, ml/weights/rounded1] (sig1)  -- node[left] {$f_t$} (hadmard1);
    \draw [->, ml/weights/rounded1] (sig2)  |-node[below right] {$i_t$} (hadmard2);
    \draw [->, ml/weights/rounded1] (sig3)  |- (hadmard3);
    \draw [->, ml/weights/rounded1] (add1) -| (tanh2) ;
    \draw [->, ml/weights/rounded1] (tanh2) -- (hadmard3);
    \draw [ml/weights/rounded1] (hadmard3) |- (ht);

    \draw (ct -| yt) ++ (0,-0.1) coordinate (i1);
    \draw (i1) ++ (0,0.2) -- (yt);
    \draw [-, ml/weights/rounded1] (ht -| yt)++(-0.5,0) -| (i1);

    \node[cell, fit=(xt |- htprev) (ct -| yt) (hadmard1)] {};
\end{tikzpicture}
    \caption{The Long Short-Term Memory unit internal structure from \autoref{eqn:lstm}}
    \label{fig:rae:lstm_unit}
\end{figure}
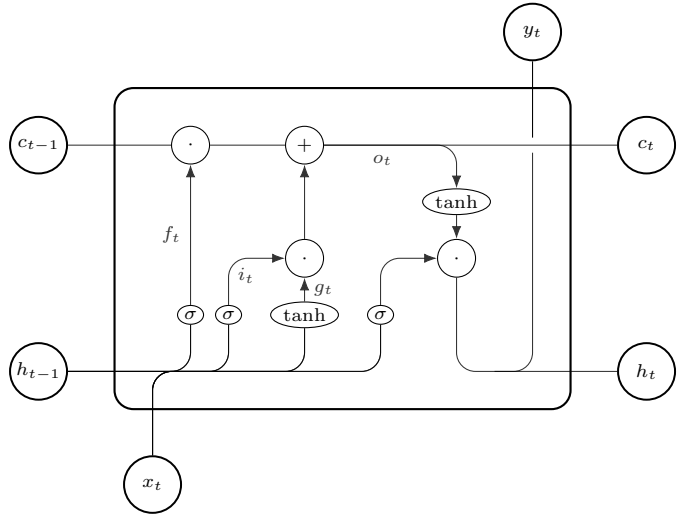

An approximation of the input is then made by the decoder $D_\theta$, parameterized by $\theta$, of the same length. The decoder is passed the latent space $z_t$ as many times as there are rows in $x_t$, as well as the hidden state $h_t$ and cell state $c_t$. Each output is then the estimated value of the corresponding item in the input sequence.
The loss is calculated using the mean square error
\begin{equation}
    \mathcal{L}(s, \hat s, z) = \frac{1}{dim(s)}\sum_{i=1}^N (\hat{s}_i - s_i)^2
\end{equation}

The optimal RAE for $s_\mathrm{path}$ encoding is then found by
\begin{gather}
    \phi^*, \theta^* = \arg \min_{\phi, \theta} \mathcal{L} \left[ s_\mathrm{path}, D_\theta(E_\phi(s_\mathrm{path})), E_\phi(s_\mathrm{path}) \right]
\end{gather}

The RAE network can be seen in \autoref{fig:rae:lstm_ae}. Using an unbalanced setup in favour of the decoder gives a higher reconstruction potential, which leads to better encoder training. The larger decoder compensates for any loss of information during encoding, ensuring that even a suboptimal latent representation can still result in high-quality training updates. This setup also stabilizes training and accelerates convergence by allowing the decoder's higher capacity to handle complex reconstruction tasks effectively.

\begin{figure*}
    \centering
    \begin{tikzpicture}[
            input/.style={circle, draw},
            output/.style={circle, draw},
            fcn/.style={rectangle, draw},
            cnn/.style={rectangle, draw},
            rnn/.style={rectangle, draw},
        ]
        {[start chain]
            \node [rnn, ml/encoder, on chain] (lstm_e) {LSTM \\ input = 2 \\ hidden = 512 \\ layers = 1};
            \node [operator, on chain] (tanh_e) {~$\tanh$~};
            \node [fcn, on chain] (layer_norm_e) {LayerNorm};
            \node [fcn, ml/encoder, on chain] (fcn_e) {FCN \\ (512, 48)};
            \node [input, on chain] (z) {$\vec z$};
            \node [input,ml/decoder, on chain] (lstm_d) {LSTM \\ input = 48 \\ hidden = 2048 \\ layers = 2};
            \node [operator, on chain] (tanh_d) {~$\tanh$~};
            \node [fcn, on chain] (layer_norm_d) {LayerNorm};
            \node [fcn, ml/decoder, on chain] (fcn_d) {FCN \\ (2048, 2)};
            \node [operator, on chain] (softsign_d) {~$\frac{x}{1+|x|}$~};
            \node [output, on chain] (x_hat) {$\hat {\vec x}$};
        }
        \coordinate (lstm_e_north_west) at ($(lstm_e.west)+(0,0.5)$);
        \coordinate (lstm_e_south_west) at ($(lstm_e.west)-(0,0.5)$);

        \node [input, left=of lstm_e_north_west] (x) {$\vec x$};
        \draw (x) -- (lstm_e_north_west);
        \node [input, left=of lstm_e_south_west] (hx) {$\vec h$};
        \draw (hx) -- (lstm_e_south_west);

        \node [cell, inner xsep=1.5mm, fit=(lstm_e) (fcn_e)] (encoder_border) {};
        \node[anchor=south west] at (encoder_border.north west) {Encoder};
        \node [cell, inner xsep=1.5mm, fit=(lstm_d) (softsign_d)] (decoder_border) {};
        \node[anchor=south west] at (decoder_border.north west) {Decoder};
    \end{tikzpicture}
    \caption{RAE architecture using LSTMs for encoding and decoding.
        Using an unbalanced architecture, with the decoder being larger than the encoder, enables higher quality reconstruction which results in faster training and better performance.
        Softsign is applied to the decoder output to ensure that values meet the $s_\mathrm{path} \in [-1,1]$ requirement whilst providing close to linear mappings in this range.
    }
    \label{fig:rae:lstm_ae}
\end{figure*}
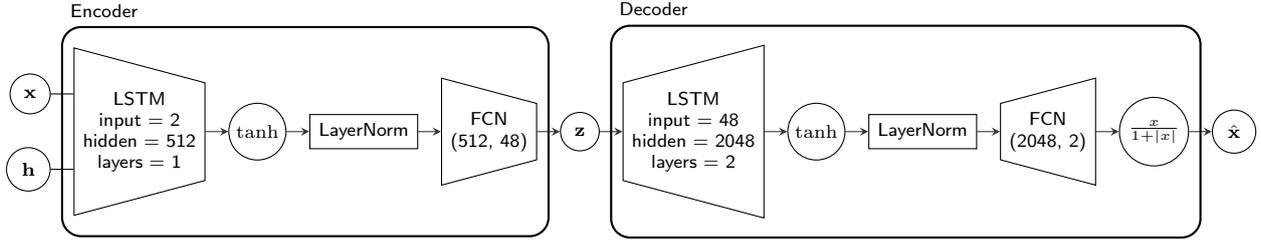

\section{Model Architectures}
\label{sect:architecture}

The principal architecture introduced in this work is the \nomdef{Long Short-Term Memory AE with Soft Actor-Critic}{\lstmaesac}. This model leverages a RAE, specifically an LSTM-based AE as previously defined in \autoref{sect:rae}, to handle path history feature extraction. The LSTMAE is coupled with the \nomdef{Soft Actor-Critic}{SAC} algorithm, which has demonstrated robust performance in continuous action spaces with large observation spaces. The \lstmaesac architecture is designed to efficiently process temporal information from the agent's path history, potentially reducing the need for an extremely large network as observed in previous work \citep{ewers_deep_2025}.

\begin{table*}[htb!]
    \centering
    \caption{Definitions of architectures}
    \label{tbl:method:definitions_of_architectures}
    \begin{tabular}{@{}l|lL{3cm}L{2.5cm}|L{2.5cm}@{}}
        \toprule
        Architecture title & RL Algorithm & Path Observation Augmentation & Inner-model Feature Extraction & Observation Space                                \\ \midrule
        \lstmaesac         & SAC          & LSTM-AE                       & None                           & $\left( z_\mathrm{path}, s_\mathrm{PDM} \right)$ \\ \midrule
        \lstmaeppo         & PPO          & LSTM-AE                       & None                           & $\left( z_\mathrm{path}, s_\mathrm{PDM} \right)$ \\
        \lstmppo           & RPPO         & None                          & None                           & $\left( s_\mathrm{pos}, s_\mathrm{PDM}\right)$   \\
        \fssaclstm         & SAC          & Frame Stacking                & LSTM                           & $\left( s_\mathrm{path},s_\mathrm{PDM}\right)$   \\
        \fssacfcn          & SAC          & Frame Stacking                & Fully Connected Network        & $\left( s_\mathrm{path},s_\mathrm{PDM}\right)$   \\
        \fssacconvtwod     & SAC          & Frame Stacking                & 2D Convolution                 & $\left( s_\mathrm{path},s_\mathrm{PDM}\right)$   \\
        \fsppolstm         & PPO          & Frame Stacking                & LSTM                           & $\left( s_\mathrm{path},s_\mathrm{PDM}\right)$   \\
        \fsppofcn          & PPO          & Frame Stacking                & Fully Connected Network        & $\left( s_\mathrm{path},s_\mathrm{PDM}\right)$   \\
        \fsppoconvtwod     & PPO          & Frame Stacking                & 2D Convolution                 & $\left( s_\mathrm{path},s_\mathrm{PDM}\right)$   \\ \bottomrule
    \end{tabular}
\end{table*}

To comprehensively evaluate the efficacy of \lstmaesac, a suite of comparative architectures were developed. These vary in their DRL algorithms, path observation augmentation techniques, and inner-model feature extraction methods. The key variants are defined in \autoref{tbl:method:definitions_of_architectures} and were designed to systematically explore the impact of different components:
\begin{itemize}
    \item Efficacy of path observation augmentation,
    \item Differences in DRL algorithm,
    \item Impact on inner-model feature extraction in lue of path observation augmentation.
\end{itemize}

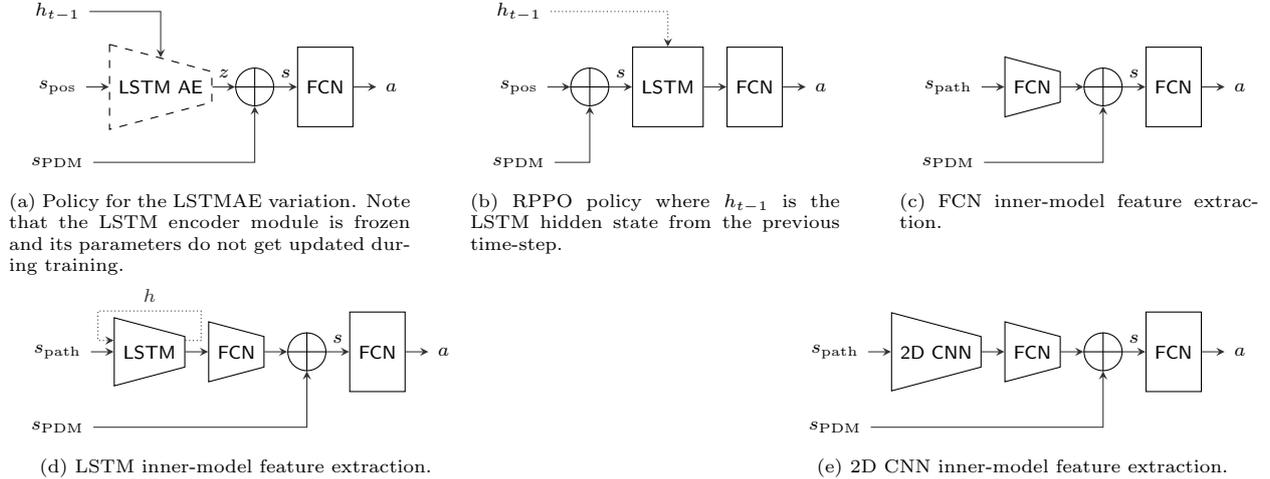
\begin{figure*}[htbp]
    \centering

    \tikzset{
        block/.style = {draw, fill=white, rectangle, minimum height=3em},
        input/.style = {fill=none, rectangle},
        output/.style= {fill=none, rectangle},
    }
    \subfloat[Policy for the LSTMAE variation. Note that the LSTM encoder module is frozen and its parameters do not get updated during training. \label{fig:method:lstmae_policy}]{
        \begin{tikzpicture}[]
            {[start chain]
                \node [input, on chain] (spath) {$s_\mathrm{pos}$};
                \node [block, ml/encoder, on chain, dashed] (lstm_e) {LSTM AE};
                \node [sum, on chain] (concatenate)  {};
                \node [block, on chain] (policy) {FCN};
                \node [output, on chain] (action) {$a$};
            }
            \node [input, above of=spath, densely dotted] (h)  {$h_{t-1}$};
            \draw [->, ml/weights/arrow] (h) -| (lstm_e);
            \node [input, below of=spath] (spdm)  {$s_\mathrm{PDM}$};
            \draw [->, ml/weights/arrow] (spdm) -| (concatenate);
            \path (concatenate) -- node[above] {$s$}(policy);
            \path (lstm_e) -- node[above] {$z$}(concatenate);
        \end{tikzpicture}
    }
    \hfill
    \subfloat[RPPO policy where $h_{t-1}$ is the LSTM hidden state from the previous time-step.\label{fig:method:rppo_policy}]{
        \begin{tikzpicture}[]
            {[start chain]
                \node [input, on chain] (spath) {$s_\mathrm{pos}$};
                \node [sum, on chain] (concatenate)  {};
                \node [block, on chain] (policy_lstm) {LSTM};
                \node [block, on chain] (policy_fcn) {FCN};
                \node [output, on chain] (action) {$a$};
            }
            \node [input, above of=spath] (h)  {$h_{t-1}$};
            \node [input, below of=spath] (spdm)  {$s_\mathrm{PDM}$};
            \draw [->, ml/weights/arrow] (spdm) -| (concatenate);
            \draw [->, ml/weights/arrow, densely dotted] (h) -| (policy_lstm);
            \path (concatenate) -- node[above] {$s$}(policy_lstm);
        \end{tikzpicture}
    }
    \hfill
    \subfloat[FCN inner-model feature extraction.\label{fig:method:fcn_policy}]{
        \begin{tikzpicture}[]
            {[start chain]
                \node [input, on chain] (spath) {$s_\mathrm{path}$};
                \node [ml/encoder, on chain] (path_fe) {FCN};
                \node [sum, on chain] (concatenate)  {};
                \node [block, on chain] (policy) {FCN};
                \node [output, on chain] (action) {$a$};
            }
            \node [input, below of=spath] (spdm)  {$s_\mathrm{PDM}$};
            \draw [->, ml/weights/arrow] (spdm) -| (concatenate);
            \path (concatenate) -- node[above] {$s$}(policy);
        \end{tikzpicture}
    }
    \vfill
    \subfloat[LSTM inner-model feature extraction.\label{fig:method:lstm_policy}]{
        \begin{tikzpicture}[]
            {[start chain]
                \node [input, on chain] (spath) {$s_\mathrm{path}$};
                \node [ml/encoder, on chain] (path_fe_lstm) {LSTM};
                \node [ml/encoder, on chain] (path_fe_fcn) {FCN};
                \node [sum, on chain] (concatenate)  {};
                \node [block, on chain] (policy) {FCN};
                \node [output, on chain] (action) {$a$};
            }
            \node [input, below of=spath] (spdm)  {$s_\mathrm{PDM}$};
            \node[rectangle, inner sep=0.2cm, draw=none, fit=(path_fe_lstm)] (path_fe_lstm_border) {};
            \draw [->, ml/weights/arrow, densely dotted]
            (path_fe_lstm.east)++(0,0.15)
            -- ($(path_fe_lstm_border.east)+(0,0.15)$)
            -- (path_fe_lstm_border.north east)
            -- node[above] {$h$} (path_fe_lstm_border.north west)
            -- ($(path_fe_lstm_border.west)+(0,0.15)$)
            -- ($(path_fe_lstm.west)+(0,0.15)$);

            \path (concatenate) --node[above] {$s$} (policy);
            \draw [->, ml/weights/arrow] (spdm) -| (concatenate);
        \end{tikzpicture}
    }
    \hfill
    \subfloat[2D CNN inner-model feature extraction.\label{fig:method:cnn_policy}]{
        \begin{tikzpicture}[]
            {[start chain]
                \node [input, on chain] (spath) {$s_\mathrm{path}$};
                \node [ml/encoder, on chain] (path_fe_cnn) {2D CNN};
                \node [ml/encoder, on chain] (path_fe_fcn) {FCN};
                \node [sum, on chain] (concatenate)  {};
                \node [block, on chain] (policy) {FCN};
                \node [output, on chain] (action) {$a$};
            }
            \node [input, below of=spath] (spdm)  {$s_\mathrm{PDM}$};

            \path (concatenate) --node[above] {$s$} (policy);
            \draw [->, ml/weights/arrow] (spdm) -| (concatenate);
        \end{tikzpicture}
    }
    \caption{The five proposed policy architectures for use with either PPO, RPPO, or SAC. Figure \protect\subref{fig:method:rppo_policy} is only used with RPPO.}
    \label{<label>}
\end{figure*}


\section{Practical Implementation Details}
\label{sect:implementation}

\subsection{Cubature}

The integral is calculated using a cubature integration scheme \citep{cools_algorithm_1997} with constrained Delaunay triangulation \citep{chew_constrained_1987} to subdivide $H_t$ into triangles for fast computation. The use of pseudo-continuous over discrete integration has shown to greatly reduce noise in previous work \citep{ewers_enhancing_2024}. Whilst noise can be beneficial to promote exploration, this must be controllable and tunable. Reducing noise in the reward function, where cubature is being used, is critical to maximize learning efficiency else expensive techniques have to be employed \citep{wang_reinforcement_2020}.

\subsubsection{Recurrent Encoder}

\paragraph{Training}

During training the test dataset was unbatched and chunked into $N \sim \mathcal{U}(2, k)$ length sections where $k$ is the length of the longest path in the batch. If $N < k$ then the hidden states would be reused for the next section rather than resetting. This significantly improved the speed of convergence. Automatic mixed precision was used to further increase training times. A no-improvement criterion was used where training would terminate if the amount of epochs since a loss function decrease breaches a patience threshold.

\paragraph{Deployment}

The recurrent encoder, as detailed in Section \ref{sect:rae}, undergoes separate training from the DRL models with its parameters frozen during DRL training. This isolation prevents latent space divergence that could destabilize the learning during online updates. In our implementation, the encoder resides in the observation preprocessing pipeline rather than the policy network itself. This architectural choice minimizes replay buffer memory and compute requirements during training (critical for SAC's experience replay mechanism), though deployment permits alternative configurations.
The two viable deployment strategies are:
\begin{itemize}
    \item Hidden State Propagation: Stores only the encoder's hidden states (including cell state), giving a constant runtime performance per step with fixed memory usage (hidden states).
    \item  Full History Processing: Maintains complete trajectory histories, with the memory requirements and runtime performance growing linearly with the episode length.
\end{itemize}
In this work the former approach - hidden state propagation - is used. The full history processing approach would be required for other techniques such as temporal convolution networks \citep{lea_temporal_2016} or transformers \citep{vaswani_attention_2017}.

\subsection{Further Architecture Details}

The RPPO implementation used in this work is from \citet{raffin_stablebaselines3_2021} which aligns closely with \citet{pleines_generalization_2022}. The 2D convolution kernel inner-model feature extractor for \fsppoconvtwod and \fssacconvtwod is from \citet{mnih_humanlevel_2015}

\section{Results}
\label{sect:results}

\subsection{Experimental setup}
\label{sect:results:experimental_setup}

Each model was trained over \num{1e7} step with 8 workers on a local \texttt{Ubuntu 22.04} machine with a
\texttt{AMD Ryzen 9 5950X} CPU with a \texttt{NVIDIA RTX A6000} GPU and $64\unit{\giga\byte}$ of RAM. Runs terminated when invalid training updates were attempted which only happened during periods of extremely poor performance. At least five runs per architecture were undertaken with randomized starting seeds which aligns with the best practices outlined in work by \citet{agarwal_deep_2022} to ensure robust analysis for DRL results. The set of run configurations was generated and then randomized during training to avoid any possibility of unforeseen interactions.

The RAE training dataset contained \num[drop-zero-decimal]{5e4} unique paths which were generated using \lhcgwconv \citep{lin_uav_2009} with the same parameters as in this work. The recurrent encoder was trained once and then deployed for all following runs.

\subsection{Evaluation Metrics}

This section outlines the metrics used to evaluate the performance of the different architectures.

\subsubsection{Mean Step Reward}

The average reward received by the agent at each time step during an episode. A higher mean step reward indicates that the agent is making more effective decisions that lead to higher rewards at each step.

\subsubsection{Mean Rollout Episode Reward}

The average reward accumulated by the agent over a complete episode during the rollout phase. This metric reflects the overall performance of the learned policy in generating high-quality paths.

\subsubsection{Mean Episode Length}

The average number of time steps taken by the agent before the episode terminates.  A longer mean episode length generally indicates that the agent is able to explore the environment more effectively and find longer, more efficient paths.

\subsubsection{Maximum Probability Efficiency}

The highest achieved probability of reaching the target within a given search area observed during training as defined in \autoref{sect:method:rewards}. This metric directly reflects the search efficacy of the architecture, indicating how well it can find the target within the specified search space.

\subsubsection{Runtime}

The total time or steps taken by the architecture to train. This metric is crucial for practical applications, as it indicates the computational cost of training the architecture.

\subsubsection{Mean Probability Efficiency}

Similar to the maximum probability efficiency, it is the mean probability efficiency achieved during evaluation. This metric provides a comprehensive assessment of the architecture's search performance across multiple runs.

\subsubsection{Number of Parameters}

The total number of learnable parameters in the neural network architecture. This metric provides an indication of the model's complexity and computational requirements.

\subsection{Architecture}
\label{sect:results:architecture}

From \autoref{fig:results:rollout_reward_mean_over_time} it is clear that \lstmaesac, \lstmaeppo, \lstmppo, and \fsppoconvtwod had the best rollout reward performance with very stable learning curves. However, the \fssacconvtwod variant was by far the least stable and consistently crashed during training with illegal update steps. This is further corroborated by \autoref{tbl:results:architecture_performance} with \fsppoconvtwod having one of the highest mean runtime steps at \num{9.82E+06} and \fssacconvtwod having the lowest at \num{1.00E+06}.
\begin{table*}[htbp]
    \centering
    \caption{Aggregated architecture results over multiple metrics gathered at the end of a training run.}
    \label{tbl:results:architecture_performance}
    \sisetup{exponent-thresholds=-4:3,
    }
    \begin{tabular}{l|
            S[table-format=1.3]
            S[table-format=1.3]
            |
            S[table-format=2.3]
            S[table-format=2.3]
            |
            S[table-format=1.3]
            S[table-format=1.3]
            |
            S[table-format=3.3, drop-exponent=true, exponent-mode=fixed, fixed-exponent=3]
            S[table-format=3.3, drop-exponent=true, exponent-mode=fixed, fixed-exponent=3]
            |
            S[table-format=2.3, drop-exponent=true, exponent-mode=fixed, fixed-exponent=6]
            S[table-format=2.3, drop-exponent=true, exponent-mode=fixed, fixed-exponent=6]
        }
        \toprule
        Architecture   & \multicolumn{2}{R{2cm}}{Mean Step Reward} & \multicolumn{2}{R{2.5cm}}{Mean Episode Length} & \multicolumn{2}{R{2.5cm}}{Maximum Probability Efficiency} & \multicolumn{2}{R{2.5cm}}{ Runtime [s/\num[exponent-mode=input, output-exponent-marker=]{e3}]} & \multicolumn{2}{R{2.5cm}}{Runtime [steps/\num[exponent-mode=input, output-exponent-marker=]{e6}]}                                                                            \\

                       & {Mean.}                                   & {Std.}                                         & {Mean.}                                                   & {Std.}                                                                                         & {Mean.}                                                                                           & {Std.}   & {Mean.}            & {Std.}   & {Mean.}            & {Std.}   \\
        \midrule
        \fsppoconvtwod & 3.29E-01                                  & 5.19E-02                                       & 3.69E+01                                                  & 1.19E+01                                                                                       & 1.02E-01                                                                                          & 1.88E-02 & 5.11E+04           & 7.71E+03 & 9.82E+06           & 6.82E+05 \\
        \fsppofcn      & 2.72E-01                                  & 9.37E-02                                       & 2.16E+01                                                  & 1.08E+01                                                                                       & 6.59E-02                                                                                          & 1.24E-02 & 4.58E+04           & 1.54E+04 & 9.36E+06           & 1.67E+06 \\
        \fsppolstm     & 2.73E-01                                  & 1.15E-01                                       & 1.64E+01                                                  & 7.39E+00                                                                                       & 4.37E-02                                                                                          & 8.88E-03 & 9.48E+04           & 2.49E+04 & 9.87E+06           & 5.91E+05 \\
        \fssacconvtwod & 2.26E-01                                  & 1.03E-01                                       & 1.01E+01                                                  & 1.63E+00                                                                                       & 3.63E-02                                                                                          & 2.95E-02 & \bfseries 9.77E+03 & 4.55E+03 & 1.00E+06           & 3.16E+05 \\
        \fssacfcn      & 2.07E-01                                  & 9.77E-02                                       & 2.21E+01                                                  & 1.92E+01                                                                                       & 7.90E-02                                                                                          & 2.50E-02 & 7.99E+04           & 5.27E+04 & 7.23E+06           & 3.72E+06 \\
        \fssaclstm     & 1.36E-01                                  & 1.14E-01                                       & 3.20E+01                                                  & 2.51E+01                                                                                       & 8.43E-02                                                                                          & 1.67E-02 & 2.80E+05           & 1.25E+05 & 7.10E+06           & 2.99E+06 \\
        \lstmaeppo     & 5.20E-01                                  & 6.10E-02                                       & 5.11E+01                                                  & 4.25E+00                                                                                       & 1.76E-01                                                                                          & 6.71E-03 & 1.56E+05           & 4.99E+04 & \bfseries 1.00E+07 & 0.00E+00 \\
        \lstmaesac     & \bfseries 5.37E-01                        & 4.31E-02                                       & \bfseries 5.73E+01                                        & 6.09E+00                                                                                       & \bfseries 1.92E-01                                                                                & 1.16E-02 & 2.55E+05           & 5.51E+04 & 9.75E+06           & 6.10E+05 \\
        \lstmppo       & 4.13E-01                                  & 5.12E-02                                       & 5.54E+01                                                  & 7.40E+00                                                                                       & 1.64E-01                                                                                          & 1.95E-02 & 3.55E+05           & 1.28E+05 & 8.26E+06           & 3.12E+06 \\
        \bottomrule
    \end{tabular}
\end{table*}

\begin{figure}
    \centering
    \includegraphics[width=0.7\linewidth]{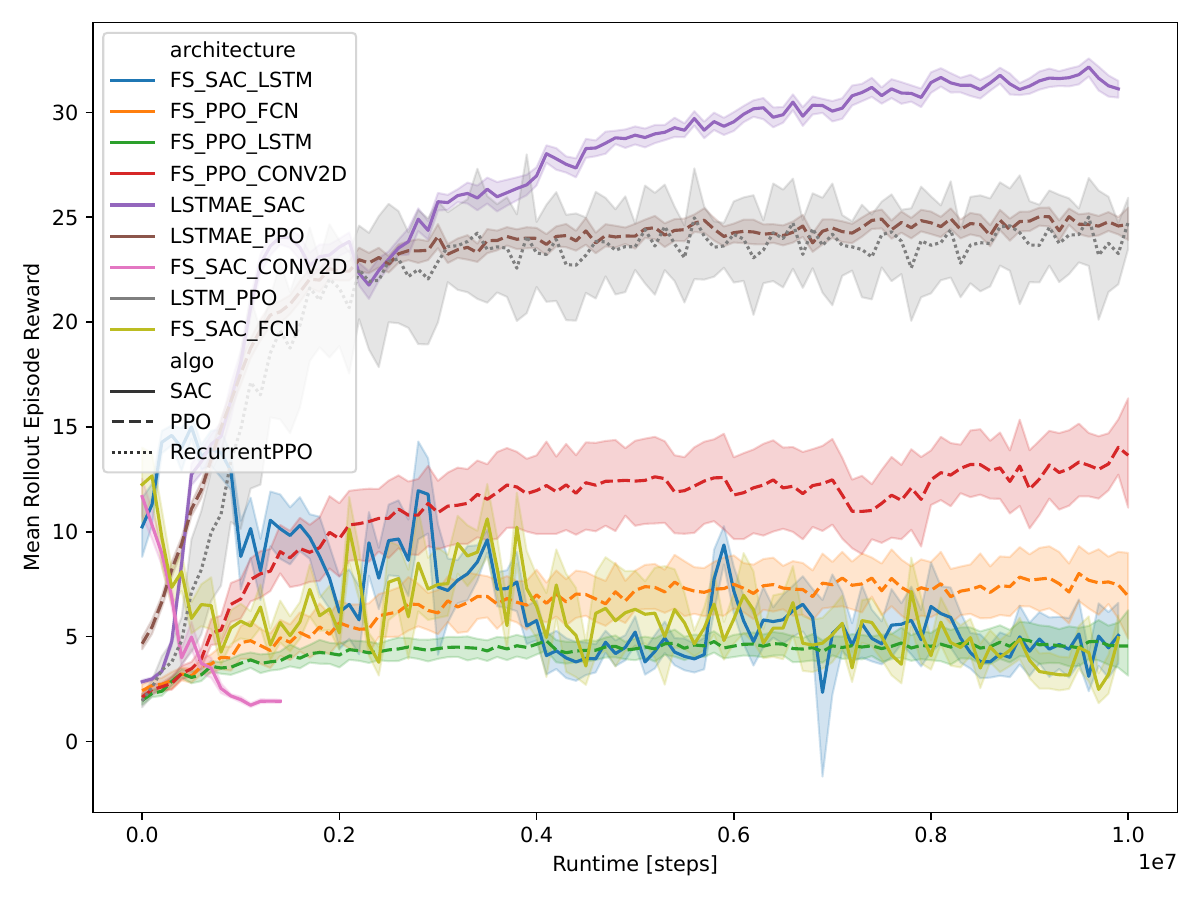}
    \caption{Mean rollout episode reward over global step for all architectures highlighting the training stability provided by the LSTMAE.}
    \label{fig:results:rollout_reward_mean_over_time}
\end{figure}

With the simulation terminating prematurely if the agent steps out-of-bounds,
it is important for the mean episode length to be as close to the maximum simulation length as possible.
The worst performance was again \fssacconvtwod with \num{1.01E+01} showing a complete lack of generalization.
\lstmaesac had the highest mean episode length at \num{5.73E+01} with \lstmaeppo and \lstmaeppo achieving similar results at \num{5.11E+01} and \num{5.54E+01} respectively.
None of the frame-stacking variants were able to breach the \num{4.00E+01} barrier with \fssacconvtwod coming closest at \num{3.69E+01} which aligns with the performance displayed in \autoref{fig:results:rollout_reward_mean_over_time}.

Mean step reward and maximum probability efficiency should strongly correlate from the definition of reward in \autoref{eqn:reward_with_cases}. However, the latter is the key metric as it directly relates to the search efficacy.
Similar to previous metrics, \fssacconvtwod displayed the poorest performance with a score of \num{3.63E-02}. \fsppoconvtwod, again, was the best of the frame-stacking variants with a score of \num{1.02E-01}. \lstmaeppo and \lstmppo also performed inline with previous results with \num{1.76E-01} and \num{1.64E-01} respectively. Whilst \lstmaeppo has a higher mean maximum probability efficiency, it also has a lower standard deviation of \num{6.71E-03} compared to \num{1.95E-02} across more runs. This could indicate the additional stability offered by having a static LSTMAE compared to the internalized LSTM module of \lstmppo from \autoref{fig:method:rppo_policy}. Furthermore, this could hint at the vanishing gradient problem that recurrent networks suffer from.

Ultimately, \lstmaesac showed the best performance across mean step reward (\num{1.70E-02} higher than \lstmaeppo),
mean episode length (\num{1.90E-02} higher than \lstmppo),
and maximum probability efficiency with a value of \num{1.92E-01} (\num{1.60E-02} higher than \lstmaeppo).
Contrary to the runtime results from \autoref{tbl:results:architecture_performance},
\autoref{fig:results:rollout_reward_mean_over_time} shows a training curve that had not reached its maxima whilst starkly outperforming the competing variants.
\subsection{Ablation Study: Reinforcement Algorithm}
\label{sect:results:algorithm}

SAC and PPO have been extensively compared in the literature, however, it is typically only the environment that is being changed. In this comparison the result from the various architectures are grouped by DRL algorithm with a consistent environment and hyperparameters. This will give an indication of the robustness and sensitivity to network changes.

\autoref{fig:results:rollout_reward_mean_over_time} displays the inter-quartile range of the maximum probability efficiency from \autoref{tbl:results:architecture_performance}. This shows that neither PPO nor SAC have a clear advantage for the frame-stacking variant. However, the LSTMAE results clearly show that SAC outperforms PPO here. The aggregated mean for PPO is \num{1.43E-01} compared to \num{1.37E-01} for SAC showing that PPO perhaps has a slight edge. However, applying a p-test gives a p-value of 0.61 which implies that the results are not conclusive.

Isolating the top frame-stacking (\fsppoconvtwod and \fssaclstm) architectures results in a different outcome. SAC has a higher aggregated mean maximum probability efficiency of \num{1.54E-01} whilst PPO only achieves \num{1.43E-01}. A p-test, however, also reveals that the the results are not conclusive with a p-value of 0.48.

\begin{table*}[htbp]
    \centering
    \caption{Aggregated algorithm (PPO, RPPO, SAC) results over multiple metrics gathered at the end of a training run.}
    \label{tbl:results:rl_algorithms}
    \sisetup{exponent-thresholds=-3:3}
    \begin{tabular}{l|
            S[table-format=1.3]
            S[table-format=1.3]
            |
            S[table-format=2.3]
            S[table-format=1.3]
            |
            S[table-format=1.3]
            S[table-format=1.3]
            |
            S[table-format=3.3, drop-exponent=true, exponent-mode=fixed, fixed-exponent=3]
            S[table-format=3.3, drop-exponent=true, exponent-mode=fixed, fixed-exponent=3]
            |
            S[table-format=1.3, drop-exponent=true, exponent-mode=fixed, fixed-exponent=6]
            S[table-format=1.3, drop-exponent=true, exponent-mode=fixed, fixed-exponent=6]
        }
        \toprule
        Algorithm & \multicolumn{2}{R{2.5cm}}{Mean Step Reward} & \multicolumn{2}{R{2.5cm}}{Mean Episode Length} & \multicolumn{2}{R{2.5cm}}{Maximum Probability Efficiency} & \multicolumn{2}{R{2.5cm}}{ Runtime [s/\num[exponent-mode=input, output-exponent-marker=]{e3}]} & \multicolumn{2}{R{2.5cm}}{Runtime [steps/\num[exponent-mode=input, output-exponent-marker=]{e6}]}                                                                            \\
                  & {Mean}                                      & {Std.}                                         & {Mean}                                                    & {Std.}                                                                                         & {Mean}                                                                                            & {Std.}   & {Mean}             & {Std.}   & {Mean}             & {Std.}   \\
        \midrule
        PPO       & \bfseries 3.53E-01                          & 1.34E-01                                       & 3.17E+01                                                  & 1.65E+01                                                                                       & 9.82E-02                                                                                          & 5.40E-02 & 9.08E+04           & 5.38E+04 & \bfseries 9.78E+06 & 9.05E+05 \\
        SAC       & 3.14E-01                                    & 1.90E-01                                       & \bfseries 3.42E+01                                        & 2.36E+01                                                                                       & \bfseries 1.12E-01                                                                                & 6.57E-02 & \bfseries 1.67E+05 & 1.29E+05 & 6.89E+06           & 3.85E+06 \\
        \midrule
        RPPO      & 4.13E-01                                    & 5.12E-02                                       & 5.54E+01                                                  & 7.40E+00                                                                                       & 1.64E-01                                                                                          & 1.95E-02 & 3.55E+05           & 1.28E+05 & 8.26E+06           & 3.12E+06 \\
        \bottomrule
    \end{tabular}
\end{table*}

\begin{figure}[htbp]
    \centering
    \includegraphics[width=0.7\linewidth]{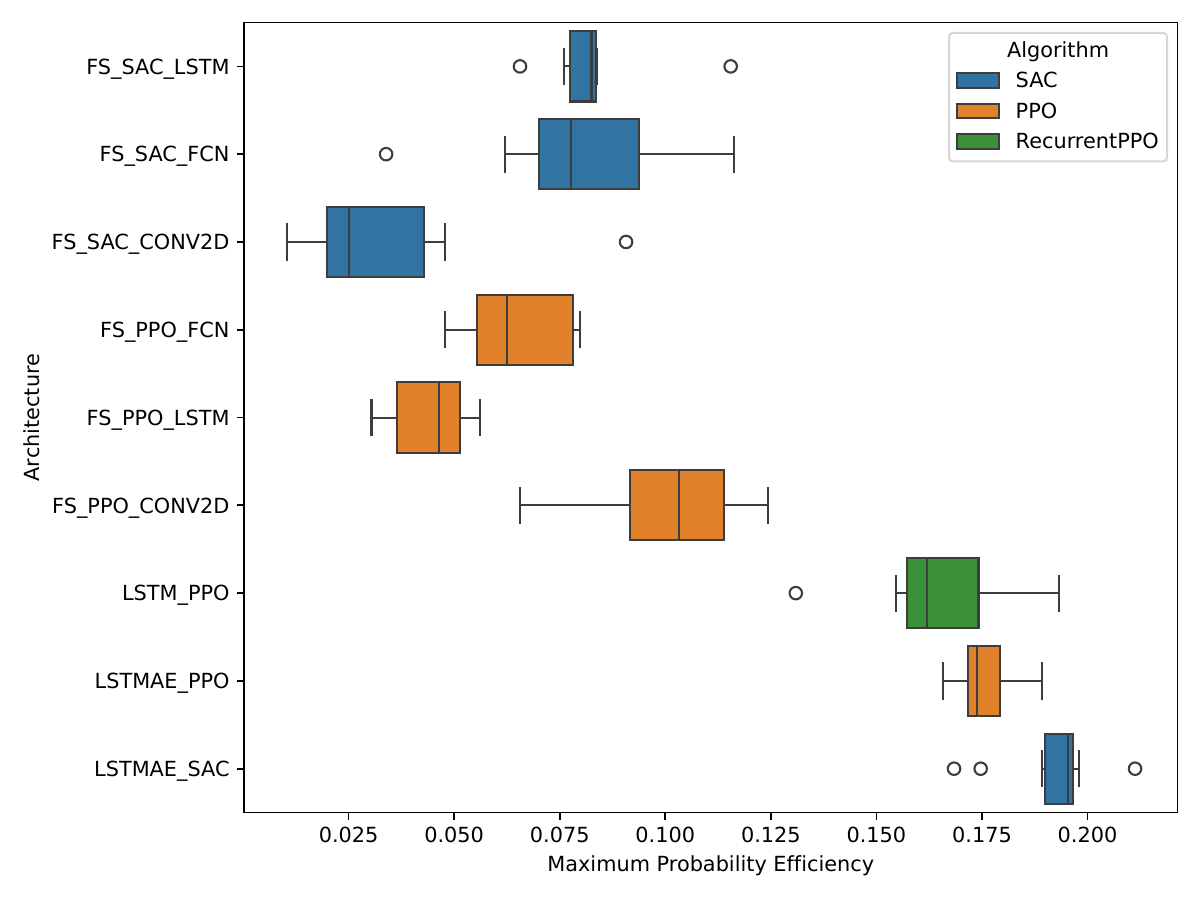}
    \caption{Maximum probability efficiency during training for all architecture variants with colour coded by DRL algorithm.}
    \label{fig:results:maximum_probability_efficiency}
\end{figure}

\subsection{Ablation Study: Path Feature Extraction}

The main contribution of this work is the use of the LSTMAE to handle the feature extraction. It is therefore imperative to evaluate its efficacy compared to the much simpler frame-stacking. \autoref{tbl:results:path_feature_extraction} shows the aggregated performance of the various feature extractors which closely align to those from \autoref{sect:results:architecture}. LSTMAE\_* is the most performant path feature extraction methods across all key performance metrics. All frame-stacking variants have low mean episodes lengths and low mean step reward. Most importantly, the mean maximum probability efficiency for the best frame-stacking variant, FS\_*\_CONV2D, is less than half as performant as LSTMAE\_*.

As expected from \autoref{sect:results:algorithm}, LSTMAE\_* also outperforms LSTM\_*  in all metrics from \autoref{tbl:results:path_feature_extraction}. This further highlights the efficiency of the RAE to capture the relevant information before DRL training to allow the DRL policy to focus on control.

\begin{table*}[htbp]
    \centering
    \caption{Aggregated path feature extraction results over multiple metrics gathered at the end of a training run.}
    \label{tbl:results:path_feature_extraction}
    \sisetup{exponent-thresholds=-3:3,table-format=2.3}
    \begin{tabular}{l|SS|SS|SS}
        \toprule
        Path Feature Extraction & \multicolumn{2}{R{2.5cm}}{Mean Step Reward} & \multicolumn{2}{R{2.5cm}}{Mean Episode Length} & \multicolumn{2}{R{2.5cm}}{Maximum Probability Efficiency}                                            \\
                                & {Mean}                                      & {Std.}                                         & {Mean}                                                    & {Std.}   & {Mean}             & {Std.}   \\
        \midrule
        FS\_*\_CONV2D           & 2.88E-01                                    & 8.96E-02                                       & 2.62E+01                                                  & 1.63E+01 & 7.60E-02           & 4.05E-02 \\
        FS\_*\_FCN              & 2.42E-01                                    & 9.84E-02                                       & 2.18E+01                                                  & 1.48E+01 & 7.20E-02           & 1.99E-02 \\
        FS\_*\_LSTM             & 2.25E-01                                    & 1.30E-01                                       & 2.19E+01                                                  & 1.71E+01 & 5.80E-02           & 2.32E-02 \\
        LSTMAE\_*               & \bfseries 5.54E-01                          & 6.52E-02                                       & \bfseries 5.62E+01                                        & 6.50E+00 & \bfseries 2.01E-01 & 3.04E-02 \\
        LSTM\_*                 & 4.13E-01                                    & 5.12E-02                                       & 5.54E+01                                                  & 7.40E+00 & 1.64E-01           & 1.95E-02 \\
        \bottomrule
    \end{tabular}
\end{table*}

\subsection{Benchmark}
\label{sect:results:benchmark}

Previous work has shown that \fssacfcn with a large enough network can outperform optimisation-based algorithms from the literature in \citep{ewers_deep_2025}. A key limitation of this approach was the large number of parameters and long training times required to achieve these results. The results from \autoref{sect:results:architecture} were achieved using a core policy of dimensions $2 \times 256$ whereas \citep{ewers_deep_2025} used $8 \times 2000$. For a fairer comparison, a $2 \times 2000$ version of \lstmaesac was also trained.

\autoref{tbl:results:long_runs} shows the mean probability efficiency from 2000 generated paths. It is evident that the smaller variants of both \fssacfcn and \lstmaesac performed poorer than their larger counterparts. However, the smaller \lstmaesac is only worse than the large \fssacfcn by a margin of \num{2.40E-02} whilst having 0.27\% the number of learnable parameters. Large \lstmaesac, which has 14.0\% of the amount of learnable parameters of large \fssacfcn, clearly outperforms all others variants. A performance difference of \num{4.00E+02} might hint at a insignificant result, however the sample sizes were above 2000 resulting in a p-value of 0.0151 which is above the threshold of 0.05 showing a meaningful difference in distributions. Large \lstmaesac therefore outperforms large \fssacfcn whilst being only 14\% of the size and only needing to be trained for 23\% of the time (90 days for \cite{ewers_deep_2025}) and 21 days for large \lstmaesac before it reached the no-improvement termination criterion).

\begin{table*}[htbp]
    \centering
    \caption{Mean probability efficiency from over 2000 generated paths by the respective architectures.}
    \label{tbl:results:long_runs}
    \sisetup{exponent-thresholds=-3:3}
    \begin{tabular}{@{}lSSrS@{}}
        \toprule
        Architecture                      & {Mean}             & {Std. }  & Core Policy       & {\# Parameters} \\
        \midrule
        \lstmaesac                        & \bfseries 2.00E-01 & 5.70E-02 & $2 \times 2000 $  & 2.17E+07        \\
        \lstmaesac                        & 1.72E-01           & 6.17E-02 & $2 \times 256$    & 4.25E+05        \\
        \midrule
        \fssacfcn \citep{ewers_deep_2025} & 1.96E-01           & 4.71E-02 & $ 8 \times 2000 $ & 1.55E+08        \\
        \fssacfcn                         & 9.03E-03           & 1.91E-02 & $2 \times 256$    & 4.36E+05        \\
        \midrule
        \lhcgwconv \citep{lin_uav_2009}   & 1.21E-01           & 8.02E-02 & n/a               & n/a             \\
        \bottomrule
    \end{tabular}
\end{table*}

\section{Discussion}
\label{sect:discussion}

These results demonstrate the effectiveness of the proposed \lstmaesac architecture for search planning. The LSTMAE effectively captures temporal dependencies within the path, leading to improved performance and stability compared to simpler frame-stacking methods. The RAE training harness maximizes the information throughput in the latent representation which is akin to lossy compression. This prevents the onus from being on the DRL algorithm to perform the same function. Furthermore, the recurrent network allows maximum throughput regardless of path length - from 2 steps to $N$ - since every neuron within the policy is trainable regardless of state. The training stability seen in \autoref{fig:results:rollout_reward_mean_over_time} underlines this with both \lstmaeppo and \lstmaesac having stable learning curves.

While further investigation is needed to definitively determine the optimal DRL algorithm over all architectures, the results show that SAC has an edge over PPO when used with LSTMAE. This aligns with results from \citet{mock_comparison_2023} which suggested that SAC is better for larger dimensional observation spaces.

Isolating the path feature extraction from the DRL algorithm in \autoref{sect:results:algorithm} suggested that the frame-stacking with LSTM was the poorest performer with LSTMAE being the best. This furthers provides evidence to the theory that the RAE is superior to standalone DRL. On the other hand, \lstmppo is close in performance to \lstmaeppo which suggests that simple frame-stacking is the main issues.

The significant performance improvement of the larger \lstmaesac over the larger \fssacfcn in \autoref{sect:results:benchmark} highlights the true efficiency and potential of the proposed approach. Similarly, the minor performance difference between large \fssacfcn and small \lstmaesac shows the power of the highly specialized network architecture proposed in this work. This suggests that the \lstmaesac in general can achieve high performance with fewer parameters and shorter training times, making it a promising approach for real-world applications where computational budgets can be limited.

\section{Conclusion}
\label{sect:conclusion}

In this study, we have presented a novel approach to enhancing WiSAR missions through the integration of an LSTM-based RAE with DRL. The challenges associated with traditional search planning methods, particularly in large and complex environments, necessitate solutions that can optimize both efficiency and effectiveness in locating missing persons.

Our proposed framework addresses key limitations of existing methodologies by decoupling feature extraction from the policy training phase, thereby significantly reducing the number of learnable parameters and improving training speeds. By employing a RAE architecture, we have demonstrated that it is possible to maintain high performance while also ensuring the model is lightweight enough for deployment on resource-constrained devices. The empirical results indicate that our approach enhances the probability of detection and thus accelerates the overall search process, which is critical in time-sensitive rescue scenarios.

Furthermore, this work contributes to the growing body of literature on DRL applications in search planning by providing a comprehensive evaluation of various architectures, including comparisons between PPO and SAC algorithms in large observation domains. SAC outperformed PPO for the proposed architecture whereas PPO and SAC were similar for the rest. However, \fsppoconvtwod was the best non-recurrent architecture whilst \fssacconvtwod was the worst showing that this result is application specific. These findings underscore that DRL is not a golden bullet and the importance of careful model engineering when tackling difficult problems.

\bibliography{references}

\appendix

\section{Parameters}

\begin{table}[h!]
    \centering
    \caption{Simulation parameters used for this work}
    \label{tbl:sim_parameters}
    \begin{tabular}{@{}lSl@{}}
        \toprule
        Parameter                          & {Value}         & Units         \\
        \midrule
        $\ngauss$                          & 4               &               \\
        $\vec \sigma_i$                    & {diag(500,500)} &               \\
        $x_\mathrm{min}$, $y_\mathrm{max}$ & 0               & $\unit\meter$ \\
        $x_\mathrm{min}$, $y_\mathrm{max}$ & 150             & $\unit\meter$ \\
        $\lambda$                          & 8               & $\unit\meter$ \\
        $\rbuffer$                         & 2.5             & $\unit\meter$ \\
        $N_\mathrm{waypoint}$              & 64              &               \\
        $\epsilon$                         & 0.1             &               \\
        $w_{oob}$                          & 1.0             &               \\
        $w_r$                              & 0.5             &               \\
        $w_0$                              & 0.5             &               \\
        \bottomrule
    \end{tabular}
\end{table}

\begin{table}[h!]
    \centering
    \sisetup{table-alignment-mode=none}
    \caption{RAE parameters used for this work}
    \label{tbl:rae_parameters}
    \begin{tabular}{@{}lS@{}}
        \toprule
        Parameter                                 & {Value} \\
        \midrule
        Learning rate                             & 0.0001  \\
        Gradient norm clipping                    & 0.5     \\
        L1 regularization coefficient ($\lambda$) & 0.0001  \\
        Patience                                  & 40      \\
        Batch size                                & 8       \\
        \# Epochs                                 & 5e3     \\
        \# Steps                                  & 5e4     \\
        Optimizer                                 & {adam}  \\
        \bottomrule
    \end{tabular}
\end{table}

\begin{table}[h!]
    \centering
    \newcommand{\rotateheader}[1]{\rotatebox{60}{#1}}
    \sisetup{table-alignment-mode=none}
    \caption{Reinforcement learning hyperparameters for the various architectures in this work. Default values from \cite{raffin_stablebaselines3_2021} \texttt{v2.1.0} were used if not listed in the table.}
    \label{tbl:hyperparameters}
    \begin{subtable}[t]{\linewidth}
        \caption{SAC}
        \label{tbl:hyperparameters:sac}
        \begin{tabular}{@{}l|SSSSS@{}}
            \toprule
                            & \multicolumn{5}{c@{}}{Architecture}                                                                                                                 \\
            \cmidrule(l){2-6}
            Hyperparam.     & \rotateheader{\fssacconvtwod}       & \rotateheader{\fssacfcn} & \rotateheader{\fssaclstm} & \rotateheader{\lstmaesac} & \rotateheader{ \lstmaesac} \\
            \midrule
            Net width       & 2.56e+02                            & 2.56e2                   & 25.6e1                    & 25.6e1                    & \bfseries 204.8e1          \\
            $\tau$          & 1.9736e-1                           & 1.9736E-01               & 1.9736E-01                & 1.9028E-01                & 1.9028E-01                 \\
            \# envs         & 8.0000E+00                          & 8.0000E+00               & 8.0000E+00                & 8.0000E+00                & 8.0000E+00                 \\
            Training freq.  & 1.0000E+00                          & 1.0000E+00               & 1.0000E+00                & 1.0000E+01                & 1.0000E+01                 \\
            Batch size      & 1.0240E+03                          & 1.0240E+03               & 1.0240E+03                & 5.1200E+02                & 5.1200E+02                 \\
            Buffer size     & 5.0000E+06                          & 5.0000E+06               & 5.0000E+06                & 5.0000E+05                & 5.0000E+05                 \\
            Gradient Steps  & 2.0000E+00                          & 2.0000E+00               & 2.0000E+00                & 1.0000E+02                & 1.0000E+02                 \\
            Learning rate   & 1.4363E-04                          & 1.4363E-04               & 1.4363E-04                & 5.9174E-06                & 5.9174E-06                 \\
            Learning starts & 8.0920E+03                          & 8.0920E+03               & 8.0920E+03                & 6.5470E+03                & 6.5470E+03                 \\
            Target entropy  & -1.0000E+00                         & -1.0000E+00              & -1.0000E+00               & -1.0000E+00               & -1.0000E+00                \\
            SDE freq.       & 5.0000E+00                          & 5.0000E+00               & 5.0000E+00                & 4.0000E+00                & 4.0000E+00                 \\
            Optimizer       & {adam}                              & {adam}                   & {adam}                    & {adam}                    & {adam}                     \\
            \bottomrule
        \end{tabular}
    \end{subtable}
    \vfill
    \begin{subtable}[t]{0.32\linewidth}
        \caption{RPPO}
        \label{tbl:hyperparameters:rppo}
        \begin{tabular}{@{}l|S@{}}
            \toprule
                          & \multicolumn{1}{c@{}}{Architecture} \\
            \cmidrule(l){2-2}
            Hyperparam.   & \rotateheader{\lstmppo}             \\
            \midrule
            \# steps      & 1.0247E+04                          \\
            \# envs       & 8.0000E+00                          \\
            Entropy coef. & 2.5000E-07                          \\
            $V_f$ coef.   & 2.3633E-03                          \\
            Batch size    & 1.2800E+02                          \\
            Learning rate & 8.5806E-05                          \\
            SDE freq.     & 2.0000E+00                          \\
            \# epochs     & 3.7000E+01                          \\
            Clip range    & 1.4871E-01                          \\
            GAE $\lambda$ & 9.7492E-01                          \\
            Max. grad.    & 9.8960E-01                          \\
            Optimizer     & {adam}                              \\
            \bottomrule
        \end{tabular}
    \end{subtable}
    \begin{subtable}[t]{0.62\linewidth}
        \caption{PPO}
        \label{tbl:hyperparameters:ppo}
        \begin{tabular}{@{}l|SSSS@{}}
            \toprule
                          & \multicolumn{4}{c@{}}{Architecture}                                                                                     \\
            \cmidrule(l){2-5}
            Hyperparam.   & \rotateheader{\fsppoconvtwod}       & \rotateheader{\fsppofcn} & \rotateheader{\fsppolstm} & \rotateheader{ \lstmaeppo} \\
            \midrule
            \# steps      & 6.126E+03                           & 5.126E+03                & 6.126E+03                 & 3.104E+03                  \\
            \# envs       & 8.000E+00                           & 7.000E+00                & 8.000E+00                 & 8.000E+00                  \\
            Entropy coef. & 2.070E-07                           & 1.070E-07                & 2.070E-07                 & 1.550E-02                  \\
            $V_f$ coef.   & 1.356E-02                           & 0.356E-02                & 1.356E-02                 & 3.287E-01                  \\
            Batch size    & 2.560E+02                           & 1.560E+02                & 2.560E+02                 & 6.400E+01                  \\
            Learning rate & 3.108E-04                           & 2.108E-04                & 3.108E-04                 & 4.360E-04                  \\
            SDE freq.     & 3.000E+00                           & 2.000E+00                & 3.000E+00                 & 5.000E+00                  \\
            \# epochs     & 4.000E+01                           & 3.000E+01                & 4.000E+01                 & 5.800E+01                  \\
            Clip range    & 1.482E-01                           & 0.482E-01                & 1.482E-01                 & 8.208E-02                  \\
            GAE $\lambda$ & 9.009E-01                           & 8.009E-01                & 9.009E-01                 & 9.285E-01                  \\
            Max. grad.    & 1.953E-01                           & 0.953E-01                & 1.953E-01                 & 6.142E-01                  \\
            Optimizer     & {adam}                              & {adam}                   & {adam}                    & {adam}                     \\
            \bottomrule
        \end{tabular}
    \end{subtable}
    \vfill

\end{table}

\end{document}